\documentclass[final,a4paper]{siamltex}

\usepackage{amsmath,amssymb,amsfonts}
\usepackage{color}
\usepackage{verbatim}
\usepackage{eepic,epic,epsfig,epstopdf}
\usepackage[]{graphics}
\usepackage{graphicx,inputenc}
\usepackage{subfigure}
\usepackage{changepage}
\usepackage{amsmath}
\usepackage{threeparttable}
\usepackage{extarrows}
\usepackage{url,color}
\usepackage{graphicx,subfig}
\usepackage{bm,color,url}
\usepackage{algorithm,algorithmic}
\usepackage{geometry}
\usepackage{setspace}
\usepackage{booktabs}
\usepackage{threeparttable}
\usepackage{multirow}
\usepackage[notref,notcite]{showkeys}

\usepackage{xcolor}
\colorlet{BLUE}{blue} % change color in title partly

\newtheorem{Theorem}{Theorem}[section]
\newtheorem{Lemma}[Theorem]{Lemma}
\newtheorem{Proposition}[Theorem]{Proposition}
\newtheorem{Definition}[Theorem]{Definition}
\newtheorem{Corollary}[Theorem]{Corollary}
\newtheorem{Remark}{Remark}[section]

\newcommand{\tabincell}[2]{\begin{tabular}{@{}#1@{}}#2\end{tabular}}

\def\relu{{\mathrm{ReLU}}}

\title{Deep Neural Networks with ReLU-Sine-Exponential  Activations Break  Curse of Dimensionality {\color{BLUE} in approximation} on  H\"{o}lder  Class}

\author{Yuling Jiao\thanks{School of Mathematics and Statistics, and
Hubei Key Laboratory of Computational Science, Wuhan University, Wuhan 430072, P.R. China. (yulingjiaomath@whu.edu.cn)}\quad\and
Yanming Lai \thanks{School of Mathematics and Statistics,  Wuhan University, Wuhan 430072, P.R. China. (laiyanming@whu.edu.cn)}\quad\and
 Xiliang Lu\thanks{School of Mathematics and Statistics, and
Hubei Key Laboratory of Computational Science, Wuhan University, Wuhan 430072, P.R. China. (xllv.math@whu.edu.cn)}\quad\and
Fengru Wang \thanks{School of Mathematics and Statistics,  Wuhan University, Wuhan 430072, P.R. China. (wangfr@whu.edu.cn)}\quad\and
Jerry Zhijian Yang \thanks{School of Mathematics and Statistics, and
Hubei Key Laboratory of Computational Science, Wuhan University, Wuhan 430072, P.R. China. (zjyang.math@whu.edu.cn)}\quad\and
Yuanyuan Yang \thanks{School of Mathematics and Statistics,  Wuhan University, Wuhan 430072, P.R. China. (yuanyuanyang@whu.edu.cn)}
}

\begin{document}

\maketitle

\begin{abstract}
In this paper, we construct neural networks with ReLU, sine and $2^x$ as activation functions.
For general continuous $f$ defined on  $[0,1]^d$ with continuity modulus $\omega_f(\cdot)$,  we construct  $\relu$-sine-$2^x$ networks   that enjoy   an approximation rate $\mathcal{O}\left(\omega_f(\sqrt{d})\cdot2^{-M}+\omega_{f}\left(\frac{\sqrt{d}}{N}\right)\right)$, where $M,N\in \mathbb{N}^{+}$ are the hyperparameters related to widths of the networks. As a consequence, we can construct  $\relu$-sine-$2^x$ network with the depth  $6$ and width $\max\left\{2d\left\lceil\log_2(\sqrt{d}\left(\frac{3\mu}{\epsilon}\right)^{1/{\alpha}})\right\rceil ,\right.$
$\left.2\left\lceil\log_2\frac{3\mu d^{\alpha/2}}{2\epsilon}\right\rceil+2\right\}$ that  approximates  $f\in \mathcal{H}_{\mu}^{\alpha}([0,1]^d)$ within a given tolerance $\epsilon >0$ measured in $L^p$ norm with $p\in[1,\infty)$, where $\mathcal{H}_{\mu}^{\alpha}([0,1]^d)$ denotes the  H\"{o}lder  continuous function class  defined on $[0,1]^d$ with order $\alpha \in (0,1]$ and constant $\mu > 0$.
	%For approximation in $L^{\infty}$ norm, the depth and width are $2d+5$ and $3^d\left[\max\left\{\left\lceil2d^{3/2}\left(\frac{3\mu}{\epsilon}\right)^{1/{\alpha}}\right\rceil,2\left\lceil\log_2\frac{3\mu d^{\alpha/2}}{2\epsilon}\right\rceil+2\right\}+4\right]$, respectively.
	%An important feature of,
	%achieve  super expressive power and
Therefore, the  $\relu$-sine-$2^x$ networks overcome the curse of dimensionality {\color{blue} in approximation} on $\mathcal{H}_{\mu}^{\alpha}([0,1]^d)$.
In addition to its super expressive power, functions implemented by $\relu$-sine-$2^x$ networks are (generalized) differentiable, enabling us to apply SGD to train.\\
	
\end{abstract}

%\linenumbers

\section{Introduction}
In recent years, deep learning has aroused great interest among mathematicians. How to approximate some common function classes with neural network is an important theoretical issue in this field. Some early works can be dated back to the 1980s \cite{cybenko1989approximation,hornik1989multilayer,hornik1991approximation,pinkus1999approximation}. These results are mainly focused on sigmoidal networks, i.e., the activation functions are sigmoidal functions. Recently, $\relu$ networks are attached great interest due to its superior empirical performances in nowadays learning tasks \cite{krizhevsky2012imagenet}. Comparing to sigmoidal networks, $\relu$ networks do not suffer from the vanishing gradient problem \cite{glorot2011deep}. Moreover, the $\relu$ is easy to compute and improves the ability of  data representation  \cite{bengio2013representation}. In \cite{yarotsky2017error}, Yarotsky firstly shows how to construct a $\relu$ network to achieve any approximation accuracy by the idea of Taylor expansion. Suzuki then shows that the $\relu$ networks can also be built up based on the classical approximation results of B-spline \cite{suzuki2018adaptivity}. From a different point of view, Shen et al. construct $\relu$ networks to achieve any given  accuracy by  explicitly adjusting the depths and widths   \cite{shen2019deep,lu2020deep}. Readers are also referred to some other excellent works related to ReLU networks \cite{schmidt2020nonparametric,guhring2020error,liang2016deep,lu2017expressive,gribonval2019approximation}.

Unfortunately, all those results of $\relu$ networks suffer from the curse of dimensionality {\color{blue} in approximation} \cite{donoho2000high}, which is a term commonly used to describe of the difficulty of the problem depending on
 of the  dimension exponentially. In the case of network approximation, it is usually reflected in the fact that the size of the network is exponentially dependent on the approximation error. In fact, Yarotsky already proves that $\relu$ networks cannot escape the curse of dimensionality {\color{blue} in approximation} by constructing a lower bound for network size, which is based on the VC dimension of $\relu$ networks \cite{yarotsky2017error}.

\subsection{Main Contributions}
In this paper, we construct neural networks achieving super expressive power  with ReLU, sine, and $2^x$ as activation functions.
%As far as we know, this is the first  work on
The
constructed ReLU-sine-$2^x$  networks
break   the curse of dimensionality {\color{blue} in approximation} on  H\"{o}lder  continuous function class  defined on $[0,1]^d$ and can be trained by SGD. The main contributions  of
this paper are summarized as follows.
 Let $M,N \in \mathbb{N}^{+}$ be hyperparameters related to   width,  we construct  deep networks $\Phi$ with  $\relu$-sine-$2^x$ activation functions   that enjoy  following approximation rate.
\begin{itemize}
\item For general continuous function $f$ defined on  $[0,1]^d$ with continuity modulus $\omega_f(\cdot)$, we have
    $$\|f-\Phi\|_{L^p}\leq\mathcal{O}\left(\omega_f(\sqrt{d})\cdot2^{-M}+\omega_{f}\left(\frac{\sqrt{d}}{N}\right)\right),$$  \sloppy where $p\in [1,\infty),$   the  depth $\mathcal{L}(\Phi) = 6$ and  the width $\mathcal{W}(\Phi) =$
    $\max\left\{2d\lceil\log_2N\rceil ,2M\right\}$ as in Theorem \ref{Lp general result}.
    And
$$\|f-\Phi\|_{L^{\infty}}\leq \mathcal{O}\left(\omega_{f}(\sqrt{d})\cdot2^{-M}+\omega_{f}\left(\frac{\sqrt{d}}{N}\right)\right),$$  \sloppy where the depth $\mathcal{L}(\Phi)=2d+6,$ and width $\mathcal{W}(\Phi)=3^d\left(\max\left\{2d\lceil\log_2N\rceil ,2M\right\}+4\right)$ as in Theorem \ref{pointwise general result}.
    \item
 If $f\in \mathcal{H}_{\mu}^{\alpha}([0,1]^d)$,  the  H\"{o}lder   function class
with order $\alpha \in (0,1]$ and constant $\mu > 0$, then we obtain
  $$\|f-\Phi\|_{L^p}\leq \epsilon$$ as long as $p\in[1,\infty)$,
  the depth  $\mathcal{L}(\Phi) =6$ and   width $$\mathcal{W}(\Phi)=\max\left\{2d\left\lceil\log_2(\sqrt{d}\left(\frac{3\mu}{\epsilon}\right)^{1/{\alpha}})\right\rceil ,2\left\lceil\log_2\frac{3\mu d^{\alpha/2}}{2\epsilon}\right\rceil+2\right\}.$$
It implies that the constructed $\Phi$ breaking the curse of dimensionality {\color{blue} in approximation} on $\mathcal{H}_{\mu}^{\alpha}([0,1]^d)$,
  see Corollary \ref{Lp Lip result} and  Corollary \ref{pointwise Lip result}.
%And  $$\|f-\Phi\|_{L^{\infty}}\leq \epsilon$$ provided
%   the depth $\mathcal{L}(\Phi) = 2d+5$ and width $$\mathcal{W}(\Phi) =3^d\left[\max\left\{\left\lceil2d^{3/2}\left(\frac{3\mu}{\epsilon}\right)^{1/{\alpha}}\right\rceil,2\left\lceil\log_2\frac{3\mu d^{\alpha/2}}{2\epsilon}\right\rceil+2\right\}+4\right],$$ see Corollary \ref{pointwise Lip result}.
\item
  Functions implemented by $\Phi$ are (generalized) differentiable
   \cite{clarke1990optimization,berner2019towards}, thus, they can be trained by   first order optimization algorithms such as stochastic gradient descent method.
%In this paper, we construct a  network breaking the  curse of dimensionality for  H\"{o}lder  continuous functions. We also present network size in terms of continuity modulus $\omega_f$ when approximating general continuous functions. Our network is built with $\relu$, sine, and exponential functions as activation functions, which means that functions implemented by our network are continuous and their generalized  derivative  exist
\end{itemize}

\subsection{Related Works}
To avoid curse of dimensionalityc {\color{blue} in approximation}, one needs  more regularity or structures on the target functions. For compositional functions \cite{poggio2017and}, there exists a network with  smooth, non-polynomial activation function, constant depth and width $\mathcal{O}\left(\epsilon^2\right)$ to achieve error $\epsilon$. The functions defined on low dimensional submanifolds are studied in \cite{chui2018deep,schmidt2019deep,shaham2018provable,chen2019efficient}. In \cite{shaham2018provable} it is shown that for functions in $C^2(\Gamma)$, where $\Gamma$ is a smooth $m$-dimensional manifold, there exists a $\relu$ network with depth $4$ and the number of units $\mathcal{O}\left(\epsilon^{-m/2}\right)$ to achieve error $\epsilon$. The functions with finite Fourier moment conditions are studied in \cite{barron1993universal,siegel2020approximation}. In \cite{barron1993universal} it is shown that there exists a shallow sigmoidal network with depth 2 and width $\mathcal{O}\left(\frac{1}{\epsilon}\right)$ to achieve error $\epsilon$. Smooth functions  are studied in \cite{lu2020deep,yarotsky2019phase, montanelli2020error,weinan2018exponential}. In \cite{lu2020deep} it is shown that for $f\in C^s\left([0,1]^d\right)$, to achieve an error $\mathcal{O}\left({\|f\|_{C^{s}\left([0,1]^{d}\right)} N^{-2 s / d} L^{-2 s / d}}\right)$,  the depth and width of the  $\relu$ network $\Phi$ are required to be  $\mathcal{L}(\phi) =\mathcal{O}\left((L+2) \log _{2}(4 L)+2 d\right)$ and  $\mathcal{W}(\Phi) =\mathcal{O}\left((N+2) \log _{2}(8 N)\right)$. Piecewise smooth functions are studied in \cite{liang2016deep,petersen2018optimal}. In \cite{petersen2018optimal} ReLU networks with constant depth and number of weights $\mathcal{O}\left(\epsilon^{-2(d-1)/\beta}\right)$, where $\beta$ characterizes smoothness of target functions, are constructed to achieve error $\epsilon$. For analytic functions on $(-1,1)^d$, there exists a ReLU network with depth $L$ and width $d+4$ to achieve accuracy $\mathcal{O}\left(e^{-d\delta (e^{-1}L^{1/{2d}}-1)}\right)$ for any $\delta>0$ \cite{weinan2018exponential}. For band-limited functions, there exists a $\relu$ network $\Phi$ with depth $\mathcal{L}(\Phi)=\mathcal{O}\left(\log_{2}^2(\frac{1}{\epsilon})\right)$ and width $\mathcal{W}(
\Phi)=\mathcal{O}\left(\frac{1}{\epsilon^2}\log_{2}^2(\frac{1}{\epsilon})\right)$ to achieve error $\epsilon$ \cite{montanelli2019deep}. For functions in Korobov spaces, there exists a $\relu$ network $\Phi$ with depth $\mathcal{L}(\Phi) = \mathcal{O}\left(\log _{2} \frac{1}{\epsilon }\right)$ and the number of units $\mathcal{O}\left(\frac{1}{\epsilon^2}\left(\log _{2} \frac{1}{\epsilon}\right)^{\frac{3}{2}(d-1)+1} \right)$ to achieve error $\epsilon$ \cite{montanelli2019new}. For measure $\mu$ whose support has a Minkowski dimension $d<D$, where $D$ is the ambient dimension, the approximation error measured in the norm $L^{\infty}(\mu)$ is roughly $\mathcal{O}(W^{-\beta/d})$ where $\beta$ characterizing smoothness of target functions and $W$ is  the number of parameters of the  ReLU network \cite{nakada2019adaptive}. For holomorphic mappings, the approximation rate of ReLU network is $\mathcal{O}\left(e^{-bW^{1/{(d+1)}}}\right)$ with $b$ depending on the domain of analyticity and $W$ being the number of weights \cite{opschoor2019exponential}.

Although these works have achieved great achievements, an interesting question we can still ask is that for functions without much additional regularity conditions, can we construct an approximation network which does not suffer from the curse of dimensionality {\color{blue} in approximation}?
For  H\"{o}lder  continuous functions, Shen et al. gives a positive answer by building a $\relu-$floor network overcoming curse of dimensionality {\color{blue} in approximation} \cite{shen2020deep}.
The size of their network can be adjusted by setting different values of depth and width.
For example, to approximate a  H\"{o}lder  continuous function on $[0,1]^d$ with H$\mathrm{\ddot{o}}$lder constant $\mu$ and order $\alpha$, there exists a $\relu-$floor network with depth $64d+3$ and width $\max\left\{d,5\sqrt{d}\left(\frac{3\mu}{\epsilon}\right)^{1/\alpha}+13\right\}$, where $\epsilon$ is the given approximation tolerance. However, it is a pity that the existence of floor activations exhibit using the working horse  SGD \cite{robbins1951stochastic,lecun2012efficient}  for training since the gradient vanishes  by chain rule. Note that  non-piecewise constant and continuous activation functions  have also been  proposed in   \cite{shen2020deep} in order to  use SGD to  train.

The rest of the paper is organized as follows.  In Section 2,  we give some notations and definitions.  In Section 3, we  present  details on
the construction of the ReLU-sine-$2^x$  networks with super expressive power.
We give a conclusion and a short discussion in Section 4.

\section{Notations}
The continuity modulus $\omega_f(r)$ of a function $f$ is defined as
\begin{equation*}
\omega_f(r)=\sup_{\|x-y\|_2\leq r}|f(x)-f(y)|.
\end{equation*}
 For $\mu>0$ and $\alpha\in(0,1]$, the set of  H\"{o}lder  continuous function on $[0,1]^{d}$ with constant $\mu$ and order $\alpha$ is defined by
\begin{equation*}
\mathcal{H}_{\mu}^{\alpha}([0,1]^{d})=\{f:|f(x)-f(y)|\leq \mu\|x-y\|_2^{\alpha},\quad \forall x,y\in [0,1]^{d}\}.
\end{equation*}
A function $\mathbf{f}: \mathbb{R}^{d} \rightarrow \mathbb{R}^{N_{L}}$ implemented by a neural network is defined by
\begin{equation*}
\begin{array}{l}
\mathbf{f}_{0}(\mathbf{x})=\mathbf{x},\\
\mathbf{f}_{\ell}(\mathbf{x})=\varrho_{\ell}\left(A_{\ell} \mathbf{f}_{\ell-1}+\mathbf{b}_{\ell}\right) \quad \text { for } \ell=1, \ldots, L-1, \\
\mathbf{f}=\mathbf{f}_{L}(\mathbf{x}):=A_{L}\mathbf{f}_{L-1}+\mathbf{b}_{L},
\end{array}
\end{equation*}
where $A_{\ell}\in\mathbb{R}^{N_{\ell}\times N_{\ell-1}}$,
$\mathbf{b}_{\ell}\in\mathbb{R}^{N_{\ell}}$ and the activation function $\varrho_{\ell}$ is understood to act component-wise (it is allowed that there are different activation functions in different layers). For simplicity we also use $\mathbf{f}$ to present this network. $L$ is called the depth of the network and $\max\{N_{\ell},\ell=0,\cdots,L\}$ is called the width of the network. We will use $\mathcal{L}(\mathbf{f})$ and $\mathcal{W}(\mathbf{f})$ to denote the depth and width of the neural network $\mathbf{f}$, respectively. $\sum_{\ell=1}^{L} N_{\ell} $ is called number of unites of $\mathbf{f}$ and  $\{A_{\ell},\mathbf{b}_{\ell}\}$ are called the weight parameters.
%$\varrho$ is called the activation function of the network. Note that the activation function in each neuron is not necessarily to be equal.

We now introduce the concept of VC-dimension \cite{vapnik2015uniform}, which plays an important role in the research of neural network approximation.
Let $S\subset X$ be a finite subset and $H \subset\{h: X \rightarrow\{0,1\}\}$. We define by $H_{S}:=\left\{h_{\mid S}: h \in H\right\}$ the restriction of $H$ to $S$.
\begin{Definition}\label{scatter}
	The growth function of $H$ is defined by
	$$\mathcal{G}_{H}(m):=\max \left\{\left|H_{S}\right|: S \subset X,|S|=m\right\}, \quad \text { for } m \in \mathbb{N}.$$	
\end{Definition}
It is clear that for every set $S$ with $|S|=m$, we have that $\left|H_{S}\right| \leq 2^{m}$ and hence $\mathcal{G}_{H}(m) \leq 2^{m}$. We say that
a set $S$ with $|S| = m$ for which $\left|H_{S}\right|=2^{m}$ is shattered by $H$.

\begin{Definition}
	VCdim(H) is defined to be the largest integer $m$ such
	that there exists $S\subset X$ with $|S|=m$ that is shattered by $H$. In other words,
	$$\operatorname{VCdim}(H):=\max \left\{m \in \mathbb{N}: \mathcal{G}_{H}(m)=2^{m}\right\}.$$
\end{Definition}
VC-dimension reflects the capacity of a class of functions to perform binary classification of
points. The larger VC-dimension is, the stronger the capability to perform binary classification is. For more discussion of VC-dimension, readers are referred to \cite{anthony2009neural}.

\section{Construction of network}
In this section, we give detail construction of the $\relu$-sine-$2^x$  networks that  enjoy  super expressive power and  can be trained by SGD.
We will  construct $\relu$-sine-$2^x$  networks with depth $6$ and depth $2d+6$ that  approximate functions in   $L^{p}$ norm  $p\in[1,\infty)$ and  $L^{\infty}$ norm,  respectively.
%  The activation functions in the  second layer and the third layer the  $2^x$  and sine, respectively while the activation functions in  the other layers  are the $\relu$, see, Theorem \ref{almost pointwise general result} and Theorem
%  \ref{pointwise general result} for detail of the network structures.
% \begin{Remark}
%From the proof of Theorem $\ref{almost pointwise general result}$, we know
%the activation functions  are the $\relu$  in the first and fourth  layer of  $\Psi$,   while, the activation functions in the  second layer and the third layer the  $2^x$  and sine, respectively. Similar structure also hold for
%Theorem $\ref{Lp general result}$, Collorary $\ref{almost pointwise Lip result}$ and Colloary $\ref{Lp Lip result}$.
%%For Theorem $\ref{pointwise general result}$ and Corollary $\ref{pointwise Lip result}$, activation functions in the first four layers are the same as above while those in other layers are all the $\relu$.
%\end{Remark}

{\color{blue}Inspired by Lemma 7.2 in \cite{anthony2009neural}, which shows that sine functions class enjoys an infinite VC-dimension,
we give a Lemma \ref{relu and sine vcdim} below, which plays a key role in our network construction.}

\begin{Lemma} \label{relu and sine vcdim}
{\color{blue}Define
\begin{align*}
\mathcal{N}:=&\{g:g \text{ is a neural network function with depth 3 and width 2}\\
&\text{ and its activation functions being } \relu \text{ and sine} \}.
\end{align*}
Then $\{2^i\}_{i=1}^{\infty}$ are scattered by $\mathcal{N}$, i.e., for any given $n\in \mathbb{N}^{+}$, there exist $g\in\mathcal{N}$ that  interpolates
$(2^i,b_i)$, $i=1,...,n $ with $b_i\in \{0,1\}$.}
\end{Lemma}
\begin{proof}
	The proof is based on the bit-extraction technique. For any $k\in \mathrm{N}^{+}$, we demonstrate that there exists a function $g$ in $\mathcal{N}$ scattering $\{2^i\}_{i=1}^{k}$. Let $b=\sum_{i=1}^{k}b_i2^{-i-1}+2^{-(k+2)}$, where $b_i \in \{0,1\}, i = 1,...k$. Set $x_i=2^i$, then
	\begin{equation*}
	\begin{split}
	\sin(2\pi bx_i)&=\sin\left(\pi\sum_{j=1}^{k}b_j2^{i-j}+\pi\cdot2^{i-k-1}\right)\\
	&=\sin\left(b_i\pi+\pi\cdot\sum_{j=i+1}^{k}b_j2^{i-j}+\pi\cdot2^{i-k-1}\right),
	\end{split}
	\end{equation*}
	where the second equality is due to the periodicity of sine function. Since
	\begin{equation*}
	\left(\frac{1}{2}\right)^{k}\leq\left(\frac{1}{2}\right)^{k+1-i}
	\leq\sum_{j=i+1}^{k}b_j2^{i-j}+2^{i-k-1}\leq1-\left(\frac{1}{2}\right)^{k+1-i}\leq1-\left(\frac{1}{2}\right)^{k},
	\end{equation*}
	we have
	\begin{equation*}
	\sin(2\pi bx_i)\in\left\{\begin{array}{ll}
	\left[\sin\left(\frac{1}{2^k}\right),1\right], & b_i=0,\\
	\left[-1,-\sin\left(\frac{1}{2^k}\right)\right], &  b_i=1.
	\end{array}\right.
	\end{equation*}
	Define
	\begin{equation*}
	\begin{split}
	f(x)&=\relu\left(\frac{1}{2\sin(1/2^k)}x+\frac{1}{2}\right)-\relu\left(\frac{1}{2\sin(1/2^k)}x-\frac{1}{2}\right)\\
	&=\left\{\begin{array}{ll}
	1, & x>\sin\left(\frac{1}{2^k}\right), \\
	\frac{1}{2\sin(1/2^k)}x+\frac{1}{2}, &  -\sin\left(\frac{1}{2^k}\right)\leq x\leq\sin\left(\frac{1}{2^k}\right),\\
	0, & x<-\sin\left(\frac{1}{2^k}\right),
	\end{array}\right.
	\end{split}
	\end{equation*}
	{\color{blue} and $g(x)=f(\sin(2\pi bx))$, then it is easy to check that
	\begin{equation*}
	g(x_i) = f(\sin(2\pi bx_i))=\left\{\begin{array}{ll}
	1, & b_i=0, \\
	0, &  b_i=1,
	\end{array}\right.\quad i=1,2,\cdots,k.
	\end{equation*}}
	The above equation and Definition \ref{scatter} imply that  $\{2^i\}_{i=1}^{k}$ are scattered by $g(x) \in \mathcal{N}$.
\end{proof}

%\begin{Remark}
%From the proof of Lemma \ref{relu and sine vcdim}, we can see that the sine   function can be replaced by any other periodical Lipschitz functions.
%%with same period as sine function.
%Hence, the sine activation function in the network $\Phi$ constructed  in  Theorem $\ref{almost pointwise general result}$, Theorem $\ref{Lp general result}$, Collorary $\ref{almost pointwise Lip result}$,  Colloary $\ref{Lp Lip result}$, and Theorem $\ref{pointwise general result}$ and Corollary $\ref{pointwise Lip result}$ can all be replaced by any periodical Lipschitz functions.% sharing the same period with sine.
%\end{Remark}

Let $N\in \mathbb{N}^{+}, \delta >0$, define by a small region
\begin{equation}\label{eqn:omega}
\begin{split}
\Omega(N,\delta,d)=&\left\{\mathbf{x} = [x_1,...,x_i,...,x_d]^T \in\Omega=[0,1]^d:\text{there exists a  coordinate } i
\right.\\
&\left.\text{ such that }x_i\in\left(\frac{j}{N}-\delta,\frac{j}{N}\right),\quad j=1,2,\cdots, N\right\}.
\end{split}
\end{equation}
We will prove the approximation to the network outside this region first and go back to this region later.

\begin{Theorem} \label{almost pointwise general result}
	Let $M,N\in \mathbb{N}^{+}$, $\delta>0$. For any $f\in C\left([0,1]^d\right)$ with maximum  $\overline{f}$ and minimum $\underline{f}$, there exists a $\relu$-sine-$2^x$ network $\Phi$ with $\mathcal{L}(\Phi)=6, \mathcal{W}(\Phi)=\max\left\{2d\lceil\log_2N\rceil ,2M\right\}$ such that for all $\mathbf{x}\in[0,1]^d$, $\underline{f}\leq\Phi(\mathbf{x})\leq\overline{f}$ and
	\begin{equation*}
	|f(\mathbf{x})-\Phi(\mathbf{x})|\leq \omega_f(\sqrt{d})\cdot2^{-M}+\omega_{f}\left(\frac{\sqrt{d}}{N}\right),\quad \forall\mathbf{x}\in[0,1]^d\backslash \Omega(N,\delta,d).
	\end{equation*}
\end{Theorem}

{\color{blue}
We list the main ideas and steps before the complete proof. The domain $[0,1]^d\backslash\Omega(N,\delta,d)$ is divided into some uniform small cubes $\{\Omega_{\alpha}\}_{\alpha}$ with size parameter $N$. We will construct an approximation network $\Phi$ which is constant on each $\Omega_{\alpha}$. It is enough to approximate $f$ at grid points of the cubes $\{\Omega_{\alpha}\}_{\alpha}$, then approximation on  $[0,1]^d\backslash\Omega(N,\delta,d)$ can be obtained by using continuity modulus and controlling the size of $N$. To see that we first construct two maps $\Phi_1$ and $\Phi_2$, which serve to map each $\Omega_{\alpha}$ to a specific integer. Then we can approximate $f$ at grid points of the cubes $\{\Omega_{\alpha}\}_{\alpha}$ by applying the tool of binary representation. Specifically, we introduce $\phi_{3,j}$ to allocate the integers acquired by $\Phi_1$ and $\Phi_2$ to 0 or 1, depending on the value of the $j$th bit of binary representation of function value at the grid points. Combining them together we have a network $\Phi_3$ which can approximate $f$ at grid points of the cubes $\{\Omega_{\alpha}\}_{\alpha}$.
}

\begin{proof}
Our construction is similar to \cite{shen2020deep}. First we divide the region $[0,1]^d$ into $N^d$ small cubes with the same size. For $\mathbf{\alpha}\in\{0,1,2,\cdots,N-1\}^d$, define by
\begin{equation*}
\Omega_{\mathbf{\alpha}}(N,\delta,d)=\left\{\mathbf{x}\in[0,1]^d:
x_i\in\left[\frac{\alpha_i}{N},\frac{\alpha_i+1}{N}-\delta\right],\quad i=1,2,\cdots,N\right\}.
\end{equation*}
Then
\begin{equation*}
[0,1]^d=\bigcup_{\mathbf{\alpha}\in\{0,1,2,\cdots,N-1\}^d}\Omega_{\mathbf{\alpha}}(N,\delta,d)\bigcup \Omega(N,\delta,d).
\end{equation*}
%We first build an "almost" piecewise constant univariate
%function $\psi$  as follows
%\begin{equation*}
%\begin{split}
%\psi(x)&=\relu\left(\frac{1}{\delta}x-\frac{1}{\delta N}+1\right)-\relu\left(\frac{1}{\delta}x-\frac{1}{\delta N}\right)\\
%&=\left\{\begin{array}{ll}
%1, & x>\frac{1}{N} \\
%\frac{1}{\delta}x-\frac{1}{\delta N}+1, &  \frac{1}{N}-\delta\leq x\leq\frac{1}{N}\\
%0, & x<\frac{1}{N}-\delta
%\end{array}\right..
%\end{split}
%\end{equation*}
%Let
%\begin{equation*}
%\phi_1(x)=\sum_{i=1}^{N}\psi(x-(i-1)),
%\end{equation*}
%and
%\begin{equation*}
%\mathbf{\Phi}_1(\mathbf{x})=(\phi_1(x_1),\phi_1(x_2),\cdots,\phi_1(x_d)).
%\end{equation*}

{\color{blue}Let $N_1=\lceil\log_2N\rceil$.} We can build a network approximating the following periodical function
\begin{equation*}
{\color{blue}h(x)}=\left\{\begin{array}{ll}
1, & x\in\left[2k\cdot\frac{2^{N_1}}{N},(2k+1)\cdot\frac{2^{N_1}}{N}\right)\\
0, & x\in\left[(2k+1)\cdot\frac{2^{N_1}}{N},(2k+2)\cdot\frac{2^{N_1}}{N}\right)
\end{array}\right.,\quad{\color{blue} k=0,1,2,\cdots.}
\end{equation*}
To see that we consider
\begin{equation*}
\begin{split}
\phi_{1,1}(x)&:=\frac{1}{2\sin\delta}\relu\left(x+\sin\delta\right)-\frac{1}{2\sin\delta}\relu\left(x-\sin\delta\right)\\
&=\left\{\begin{array}{ll}
1, & x>\sin\delta,\\
\frac{1}{2\sin\delta}x+\frac{1}{2}, &  -\sin\delta\leq x\leq \sin\delta,\\
0, & x<-\sin\delta,
\end{array}\right.
\end{split}
\end{equation*}
and
\begin{equation*}
\begin{split}
\phi_{1,2}(x)&:=\phi_{1,1}\left(\sin\left(\frac{N\pi}{2^{N_1}}x+\delta\right)\right)\\
&=\left\{\begin{array}{ll}
1, & x\in\left[2k\cdot\frac{2^{N_1}}{N},(2k+1)\cdot\frac{2^{N_1}}{N}-2\delta\right],\\
0, & x\in\left[(2k+1)\cdot\frac{2^{N_1}}{N},(2k+2)\cdot\frac{2^{N_1}}{N}-2\delta\right],\\
\frac{1}{2\sin\delta}\sin\left(\frac{N\pi}{2^{N_1}}x+\delta\right)+\frac{1}{2}, &  \text{otherwise,}
\end{array}\right.
\end{split}
\end{equation*}
for ${\color{blue} k=0,1,2,\cdots}$ and $0\leq \delta\leq \frac{1}{2}$. It can be easily verified that $\phi_{1,2}$ is an approximation of the periodical function {\color{blue}$h(x)$}. We next define
\begin{equation*}
\phi_{1,3}^n(x):=\phi_{1,2}(2^nx),\quad {\textrm{for}}\; n=1,2,\cdots,N_1,
\end{equation*}
and let
\begin{equation} \label{Phi1}
\mathbf{\Phi}_1(\mathbf{x}):=(\phi_{1,3}^1(x_1),\cdots,\phi_{1,3}^{N_1}(x_1),
\phi_{1,3}^1(x_2),\cdots,\phi_{1,3}^{N_1}(x_2),\cdots,\phi_{1,3}^1(x_d),\cdots,\phi_{1,3}^{N_1}(x_d)){\color{blue}.}
\end{equation}

We claim that $\mathbf{\Phi}_1$ maps each $\Omega_{\mathbf{\alpha}}(N,\delta,d)$ to a corresponding $N_1d$-dimensional vector and for $\mathbf{\alpha}\neq\mathbf{\beta}$, $	\mathbf{\Phi_1}(\Omega_{\mathbf{\alpha}}(N,\delta,d))\bigcap \mathbf{\Phi_1}(\Omega_{\mathbf{\beta}}(N,\delta,d))=\varnothing$, which will be proved in Lemma $\ref{Phi1 property}$.

For any $\mathbf{\bar{\alpha}}\in\mathbb{R}^{N_1d}$, we define
\begin{equation*}
\Phi_2(\mathbf{\bar{\alpha}})=2^{\sum_{i=1}^{N_1d}\bar{\alpha}_i2^{i-1}+1},
\end{equation*}
then
\begin{equation*}
\Phi_2\circ\mathbf{\Phi_1}(\mathbf{x})\in\{2,4,8,\cdots\},\quad \mathbf{x}\in[0,1]^d\backslash \Omega(N,\delta,d).
\end{equation*}
Denote $\overline{f} $ and $\underline{f}$ as the maximum and minimum of $f$ in $[0,1]^d$, respectively. Define
\begin{equation*}
\widetilde{f}(\mathbf{x})=\frac{f(\mathbf{x})-\underline{f}}{\overline{f}-\underline{f}}{\color{blue}.}
\end{equation*}
It is clear that $0\leq\widetilde{f}\leq1$. For any $\mathbf{\alpha}\in\{0,1,2,\cdots,N-1\}^d$, we consider the grid points $\frac{\mathbf{\alpha}}{N}$ in $[0,1]^d$.
Then we express $\widetilde{f}$ in the following form of binary decomposition, {\color{blue}that is, for $\alpha\in\{0,1,\cdots,N-1\}^d$, there exists $\{a_{i_{\alpha}j}\}_{j=1}^{\infty}$ with $a_{i_{\alpha}j}\in\{0,1\}$ such that
\begin{equation*}
\widetilde{f}\left(\frac{\mathbf{\alpha}}{N}\right)=\sum_{j=1}^{\infty}a_{i_{\alpha}j}2^{-j},
\end{equation*}}
where {\color{blue}
\begin{align*}
i_{\alpha}=\sum_{k=1}^{N_1d}\left(\mathbf{\Phi}_1\left(\frac{\alpha}{N}\right)\right)_k2^{k-1}+1\in\{1,2,\cdots,N^d\}.
\end{align*}}
For $j=1,2,\cdots,M$, \sloppy by Lemma $\ref{relu and sine vcdim}$, there exists network  $\phi_{3,j}$ with ReLU and sine activations  such that
{\color{blue}
\begin{equation*}
\phi_{3,j}\left(2^{i_{\alpha}}\right)=a_{i_{\alpha}j},\quad\alpha\in\{0,1,\cdots,N-1\}^d.
\end{equation*}}
Define by
\begin{equation*}
\Phi_3=\sum_{j=1}^{M}\phi_{3,j}2^{-j} \;\;\textrm{and} \;\; \widetilde{\Phi}=\Phi_3\circ\Phi_2\circ\mathbf{\Phi}_1.
\end{equation*}
Then for $\mathbf{x}\in\Omega_{\mathbf{\alpha}}$, we have
\begin{equation*}
\begin{split}
|\widetilde{\Phi}(\mathbf{x})-\widetilde{f}(\mathbf{x})|&\leq |\widetilde{\Phi}(\mathbf{x})-\widetilde{f}\left(\frac{\mathbf{\alpha}}{N}\right)|+|\widetilde{f}\left(\frac{\mathbf{\alpha}}{N}\right)-\widetilde{f}(\mathbf{x})|\\
&\leq \sum_{j=M+1}^{\infty}a_{ij}2^{-j}+\omega_{\widetilde{f}}\left(\frac{\sqrt{d}}{N}\right)
=2^{-M}+\omega_{\widetilde{f}}\left(\frac{\sqrt{d}}{N}\right).
\end{split}
\end{equation*}
Hence for all $\mathbf{x}\in[0,1]^d\backslash \Omega(N,\delta,d)$,
\begin{equation*}
|\widetilde{\Phi}(\mathbf{x})-\widetilde{f}(\mathbf{x})|\leq 2^{-M}+\omega_{\widetilde{f}}\left(\frac{\sqrt{d}}{N}\right).
\end{equation*}
Denoted by
\begin{equation*}
\Phi=(\overline{f}-\underline{f})\widetilde{\Phi}+\underline{f},
\end{equation*}
then we can obtain that
\begin{equation*}
\begin{split}
|\Phi(\mathbf{x})-f(\mathbf{x})|
&\leq|\overline{f}-\underline{f}||\widetilde{\Phi}(\mathbf{x})-\widetilde{f}(\mathbf{x})|\\
&\leq |\overline{f}-\underline{f}|\cdot2^{-M}+|\overline{f}-\underline{f}|\cdot\omega_{\widetilde{f}}\left(\frac{\sqrt{d}}{N}\right)\\
&=\omega_f(\sqrt{d})\cdot2^{-M}+\omega_{f}\left(\frac{\sqrt{d}}{N}\right).
\end{split}
\end{equation*}
Since $0\leq\phi_{3,j}\leq1$ for $1\leq j\leq M$, $0\leq\Phi_3\leq 1$ and hence $0\leq\widetilde{\Phi}\leq 1$.  Then $\underline{f}\leq\Phi\leq \overline{f}$.

Last, we calculate the depth and width  of $\Phi$. Obviously, for $n=1,2,\cdots,N_1$, $\mathcal{L}(\phi_{1,3}^n)=3, \mathcal{W}(\phi_{1,3}^n)=2$. Then $\mathcal{L}(\mathbf{\Phi}_1)=3, \mathcal{W}(\mathbf{\Phi}_1)=2N_1d$, and $\mathcal{L}(\Phi_2\circ\mathbf{\Phi}_1)=4, \mathcal{W}(\Phi_2\circ\mathbf{\Phi}_1)=2N_1d$, and $\mathcal{L}(\phi_{3,j})=3, \mathcal{W}(\phi_{3,j})=2$. Then,  $\mathcal{L}(\Phi)=6, \mathcal{W}(\Phi)=\max\{2N_1d,2M\}=\max\left\{2d\lceil\log_2N\rceil ,2M\right\}$.
\end{proof}

%\begin{Remark}
%From the proof of Theorem $\ref{almost pointwise general result}$, we know
%the activation are the $\relu$  in the first and third layer of  $\Psi$,   while, the activation functions in the  second layer and the fourth layer the are  $2^x$  and sine, respectively. Similar stucture also hold for
%Theorem $\ref{Lp general result}$, Collorary $\ref{almost pointwise Lip result}$ and Colloary $\ref{Lp Lip result}$,  For Theorem $\ref{pointwise general result}$ and Corollary $\ref{pointwise Lip result}$, activation functions in the first four layers are the same as above while those in other layers are all the $\relu$.
%\end{Remark}

\begin{Remark}
From the proof of Theorem $\ref{almost pointwise general result}$, we know
the activation functions of $\Phi$ are the $\relu$ in the second and fifth layer, the sine in the first and fourth layer and the $2^x$ in the third layer. The same structure also hold for Theorem $\ref{Lp general result}$, Collorary $\ref{almost pointwise Lip result}$ and Colloary $\ref{Lp Lip result}$.
%For Theorem $\ref{pointwise general result}$ and Corollary $\ref{pointwise Lip result}$, activation functions in the first four layers are the same as above while those in other layers are all the $\relu$.
See Figure  $\ref{figure}$ for the detail  on the   structure of  the constructed $\relu$-sine-$2^x$ network  $\Phi$.
\end{Remark}

\begin{figure}[H]
	\centering
	\includegraphics[width=0.6\textwidth]{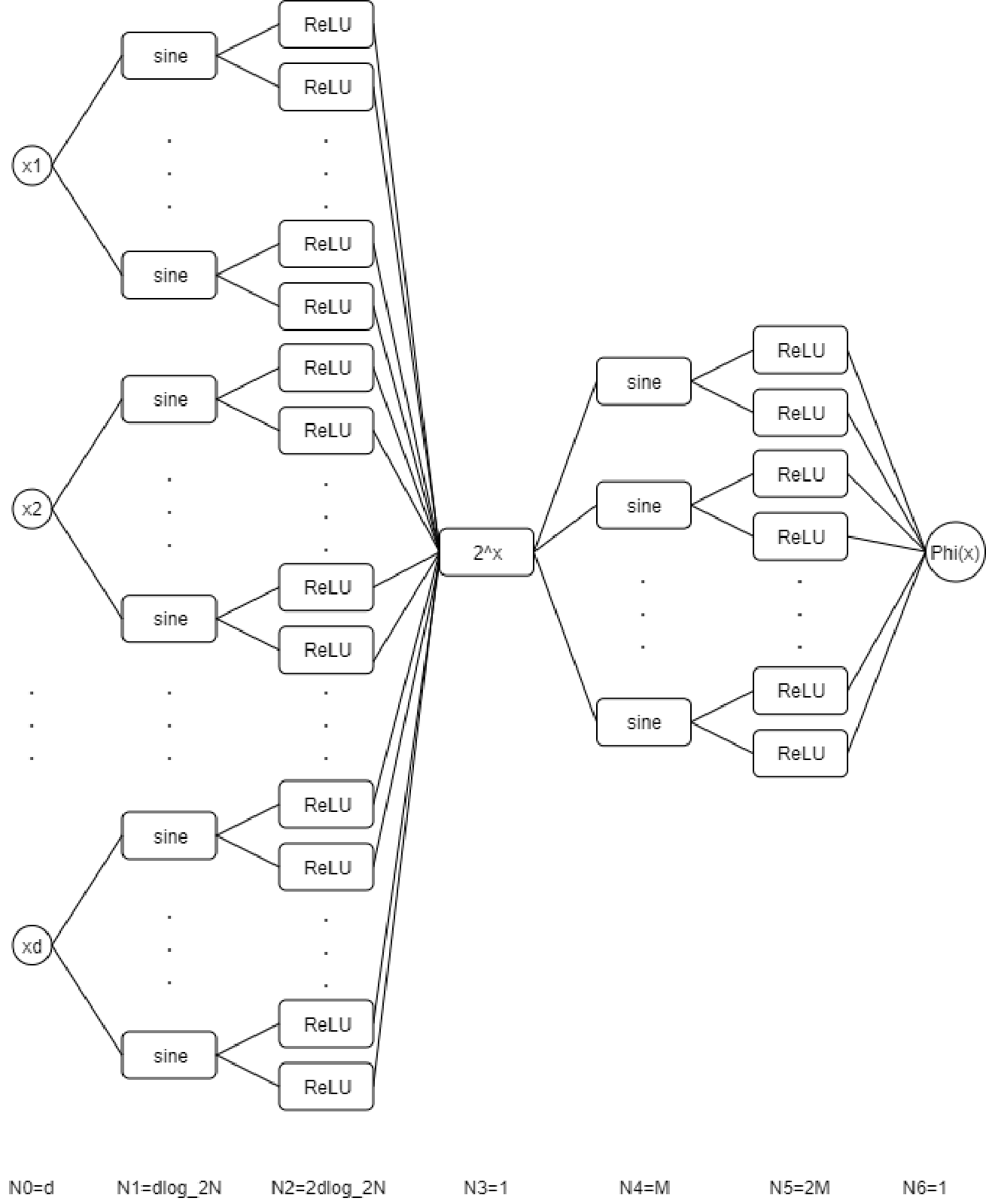}
	\caption{The neural network implementing $\Phi$}
	\label{figure}
\end{figure}

The next Lemma states properties of the mapping
 \begin{equation}
\mathbf{\Phi}_1(\mathbf{x}):=(\phi_{1,3}^1(x_1),\cdots,\phi_{1,3}^{N_1}(x_1),
\phi_{1,3}^1(x_2),\cdots,\phi_{1,3}^{N_1}(x_2),\cdots,\phi_{1,3}^1(x_d),\cdots,\phi_{1,3}^{N_1}(x_d)){\color{blue},}
\end{equation}
which have been used in the construction of network in Theorem \ref{almost pointwise general result}.

\begin{Lemma} \label{Phi1 property}
	(1) The mapping $\mathbf{\Phi_1}$ defined by $(\ref{Phi1})$ satisfies
	\begin{equation*}
	\mathbf{\Phi_1}:[0,1]^d\backslash \Omega(N,\delta,d)\to\{0,1\}^{N_1d}
	\end{equation*}
	and for each $\mathbf{\alpha}\in\{0,1,2,\cdots,N-1\}^d$, $\mathbf{\Phi_1}(\Omega_{\mathbf{\alpha}}(N,\delta,d))$ is a singleton.
	
	(2) For any $\mathbf{\alpha},\mathbf{\beta}\in\{0,1,2,\cdots,N-1\}^d$, $\mathbf{\alpha}\neq\mathbf{\beta}$,
	\begin{equation*}
	\mathbf{\Phi_1}(\Omega_{\mathbf{\alpha}}(N,\delta,d))\bigcap \mathbf{\Phi_1}(\Omega_{\mathbf{\beta}}(N,\delta,d))=\varnothing{\color{blue}.}
	\end{equation*}
\end{Lemma}
\begin{proof}
	(1) Let $\mathbf{\alpha}\in\{0,1,2,\cdots,N-1\}^d$.
	 By the definition of $\Omega_{\mathbf{\alpha}}(N,\delta,d)$ and $\mathbf{\Phi_1}$, it suffices to show that for any $i=1,2,\cdots,d$ and $n=1,2,\cdots,N_1\left(N_1=\lceil\log_2N\rceil\right)$,
	\begin{equation*}
	\left[\frac{2^n\alpha_i}{N},\frac{2^n(\alpha_i+1)}{N}-2^n\delta\right]
	\subset\left[k\cdot\frac{2^{N_1}}{N},(k+1)\cdot\frac{2^{N_1}}{N}-2\delta\right],
	\end{equation*}
	for some $k\in\mathbb{N}_0^+$. It is equivalent  to two inequalities:
	\begin{equation*}
	\left\{\begin{array}{ll}
	k\cdot\frac{2^{N_1}}{N}&\leq\frac{2^n\alpha_i}{N},\\
	(k+1)\cdot\frac{2^{N_1}}{N}-2\delta&\geq\frac{2^n(\alpha_i+1)}{N}-2^n\delta.
	\end{array}\right.
	\end{equation*}
	In the following we show that there exists a $k\in\mathbb{N}_0^+$ satisfying
	\begin{equation*}
	\frac{\alpha_i-2^{N_1-n}+1}{2^{N_1-n}}\leq k\leq\frac{\alpha_i}{2^{N_1-n}}.
	\end{equation*}
	Since $\alpha_i\in\{0,1,\cdots,N-1\}$, there exists $A_1\in\mathbb{N}_0^+$ and $A_2\in\{0,1,\cdots,2^{N_1-n}-1\}$ such that
	\begin{equation*}
	\alpha_i=A_1\cdot2^{N_1-n}+A_2.
	\end{equation*}
	Then
	\begin{equation*}
	\begin{aligned}
	A_1-\frac{\alpha_i}{2^{N_1-n}}&=\frac{-A_2}{2^{N_1-n}}&\leq0,\\
	A_1-\frac{\alpha_i-2^{N_1-n}+1}{2^{N_1-n}}&=\frac{2^{N_1-n}-1-A_2}{2^{N_1-n}}&\geq 0.
	\end{aligned}
	\end{equation*}
	Therefore, we can set $k=A_1$ and conclude the result.
	
	(2) By (1) it is sufficient to show that $\mathbf{\Phi_1}\left(\frac{\mathbf{\alpha}}{N}\right)\neq\mathbf{\Phi_1}\left(\frac{\mathbf{\beta}}{N}\right)$ for any $\mathbf{\alpha}\neq\mathbf{\beta}$.
	The fact that $\mathbf{\alpha}\neq\mathbf{\beta}$ implies there exists an index $i$, with $1\leq i\leq d$ such that $\alpha_i\neq\beta_i$. We will show there exists an index $j$, with $1\leq j\leq N_1$ such that $f_{3,j}\left(\frac{\alpha_i}{N}\right)\neq f_{3,j}\left(\frac{\beta_i}{N}\right)$. It can be verified by  contradiction. Assume that for all $1\leq n\leq N_1$ there holds $f_{3,j}\left(\frac{\alpha_i}{N}\right)= f_{3,j}\left(\frac{\beta_i}{N}\right)$, which means $f_2\left(\frac{2^n\alpha_i}{N}\right)=f_2\left(\frac{2^n\beta_i}{N}\right)$ for all $1\leq n\leq N_1$ by definition. For $n=1$, $f_2\left(\frac{2\gamma}{N}\right)=1$ when  $\gamma\in\{0,1,\cdots,2^{N_1-1}-1\}$ and $f_2\left(\frac{2\gamma}{N}\right)=0$ when $\gamma\in\{2^{N_1-1},2^{N_1-1}+1,\cdots,N-1\}$, respectively. Then we have $\alpha_i,\beta_i\in\{0,1,\cdots,2^{N_1-1}-1\}$. Similarly for $n=2$,
we can deduce that
$f_2\left(\frac{2^2\gamma}{N}\right)=1$ when  $\gamma\in\{0,1,\cdots,2^{N_1-2}-1\}$ and $f_2\left(\frac{2^2\gamma}{N}\right)=0$ when $\gamma\in\{2^{N_1-2},2^{N_1-2}+1,\cdots,2^{N_1-1}-1\}$, respectively. Thus we obtain that $\alpha_i,\beta_i\in\{0,1,\cdots,2^{N_1-2}-1\}$. The same argument can be applied to $n=N_1$, one may find $\alpha_i,\beta_i\in\{0\}$. This  contradicts to the fact that $\alpha_i\neq\beta_i$.
\end{proof}

\begin{Theorem} \label{Lp general result}
\sloppy Let $M,N\in \mathbb{N}^{+}$, $\delta>0$, $p\in[1,+\infty)$. For any $f\in C\left([0,1]^d\right)$, there exists a $\relu$-sine-$2^x$  network $\Phi$ with $\mathcal{L}(\Phi)=6, \mathcal{W}(\Phi)=\max\left\{2d\lceil\log_2N\rceil ,2M\right\}$ such that
\begin{equation*}
\|f-\Phi\|_{L^p}\leq\omega_f(\sqrt{d})\cdot2^{-M}+\omega_{f}\left(\frac{\sqrt{d}}{N}\right)+\omega_f(\sqrt{d})\left[1-(1-N\delta)^d\right]^{1/p}.
\end{equation*}
\end{Theorem}
\begin{proof}
By the approximation results in Theorem $\ref{almost pointwise general result}$, we can compute the error in $L^p$ norm as follows:
\begin{equation*}
\begin{split}
&\int_{[0,1]^d}|f-\Phi|^pd\mathbf{x}=\int_{[0,1]^d\backslash \Omega(N,\delta,d)}|f-\Phi|^pd\mathbf{x}
+\int_{\Omega(N,\delta,d)}|f-\Phi|^pd\mathbf{x}\\
&\leq \left[\omega_f(\sqrt{d})\cdot2^{-M}+\omega_{f}\left(\frac{\sqrt{d}}{N}\right)\right]^p
\int_{[0,1]^d\backslash \Omega(N,\delta,d)}d\mathbf{x}
+\omega_f^p(\sqrt{d})\int_{\Omega(N,\delta,d)}d\mathbf{x}\\
&=\left[\omega_f(\sqrt{d})\cdot2^{-M}+\omega_{f}\left(\frac{\sqrt{d}}{N}\right)\right]^p
(1-N\delta)^d
+\omega_f^p(\sqrt{d})\left[1-(1-N\delta)^d\right]\\
&\leq\left\{\omega_f(\sqrt{d})\cdot2^{-M}+\omega_{f}\left(\frac{\sqrt{d}}{N}\right)+\omega_f(\sqrt{d})\left[1-(1-N\delta)^d\right]^{1/p}\right\}^p.
\end{split}
\end{equation*}
\end{proof}

In \cite{lu2020deep}, an approach of expanding the approximation result from $\mathbf{x}\in[0,1]^d\backslash\Omega(N,\delta,d)$ to the whole region $[0,1]^d$ is developed, which is based on a technique called  horizontal shift. The result obtained in \cite{lu2020deep} is stated as follows.
\begin{Proposition}[Theorem 2.1, \cite{lu2020deep}] \label{expand to the whole region}
Given any $\epsilon>0$, $N, L, K\in\mathbb{N}^+$, and $\delta\in(0, \frac{1}{3K}]$, assume $f\in C\left([0, 1]^d\right)$
and $\widetilde{\phi}$ is a  network with width $ \mathcal{W}(\widetilde{\phi}) = N$ and depth $\mathcal{L}(\widetilde{\phi}) = L$. If
\begin{equation*}
|f(\mathbf{x})-\widetilde{\phi}(\mathbf{x})| \leq \varepsilon, \quad \mathbf{x} \in[0,1]^{d} \backslash \Omega(K, \delta, d),
\end{equation*}
then there exists a  new network  $\phi$ with width $\mathcal{W}(\phi) = 3^d(N + 4)$
and depth $\mathcal{L}(\phi) = L + 2d$ such that
\begin{equation*}
|f(\mathbf{x})-\phi(\mathbf{x})| \leq \varepsilon+d \cdot \omega_{f}(\delta), \quad\mathbf{x} \in[0,1]^{d}.
\end{equation*}
Moreover, the activation functions of $\phi$ are the activation functions of $\widetilde{\phi}$ and  $\relu$.
\end{Proposition}

\begin{Remark}
Note that Theorem 2.1 in \cite{lu2020deep} is applied for $\relu$ networks. However, its argument can be extended to network with any activation functions easily.
\end{Remark}

\begin{Theorem} \label{pointwise general result}
\sloppy Let $M,N\in \mathbb{N}^{+}$, $\delta>0$. For any $f\in C\left([0,1]^d\right)$, there exists a $\relu$-sine-$2^x$ network $\Phi$ with $\mathcal{L}(\Phi)= 2d+6, \mathcal{W}(\Phi)=3^d\left(\max\left\{2d\lceil\log_2N\rceil ,2M\right\}+4\right)$ such that
\begin{equation*}
|f(\mathbf{x})-\Phi(\mathbf{x})|\leq \omega_{f}(\sqrt{d})\cdot2^{-M}+\omega_{f}\left(\frac{\sqrt{d}}{N}\right)+d\cdot\omega_f(\delta),\quad \mathbf{x}\in[0,1]^d.
\end{equation*}
\end{Theorem}
\begin{proof}
This Theorem follows directly from Theorem $\ref{almost pointwise general result}$ and Proposition $\ref{expand to the whole region}$.
\end{proof}

\begin{Remark}
In Theorem $\ref{pointwise general result}$ (also  Corollary $\ref{pointwise Lip result}$ below), activation functions of  $\Phi$ in  the first and  fourth  layers are the sine and the third layer the $2^x$, while in other layers are all the $\relu$.
\end{Remark}

For  $f\in \mathcal{H}_{\mu}^{\alpha}([0,1]^{d})$, the continuity modulus $\omega_f(r)$ can be bounded by  H\"{o}lder  constant $\mu$, i.e.,  $\omega_f(r)\leq\mu r^{\alpha}.$  Hence we are able to obtain a series of more explicit approximation results that breaking the curse of dimensionality {\color{blue} in approximation}, i.e., to achieve an approximation error of $\epsilon$, the depth and the width depend on $\epsilon$ only polynomially rather than exponentially.

\begin{Corollary} \label{almost pointwise Lip result}
	\sloppy Let $\delta>0$. For any $f\in\mathcal{H}_{\mu}^{\alpha}([0,1]^d)$ and $\epsilon>0$, there exists a $\relu$-sine-$2^x$ network $\Phi$ with $\mathcal{L}(\Phi)=6, \mathcal{W}(\Phi)=\max\left\{2d\left\lceil\log_2(\sqrt{d}\left(\frac{2\mu}{\epsilon}\right)^{1/{\alpha}})\right\rceil ,2\left\lceil\log_2\frac{\mu d^{\alpha/2}}{\epsilon}\right\rceil+2\right\}$ such that
	\begin{equation*}
	|f(\mathbf{x})-\Phi(\mathbf{x})|\leq \epsilon,\quad\mathbf{x}\in[0,1]^d\backslash \Omega\left(\left\lceil\frac{2\mu\sqrt{d}}{\epsilon}\right\rceil,\delta,d\right).
	\end{equation*}
\end{Corollary}
\begin{proof}
	From Theorem $\ref{almost pointwise general result}$, there exists a $\relu-\sin-2^x$ network $\Phi$ with $L(\Phi)=6,W(\Phi)=\max\left\{2d\lceil\log_2N\rceil ,2M\right\}$ such that for $\mathbf{x}\in[0,1]^d\backslash\Omega(N,\delta,d)$,
	\begin{equation*}
	|f(\mathbf{x})-\Phi(\mathbf{x})|\leq \omega_{f}(\sqrt{d})\cdot 2^{-M}+\omega_{f}\left(\frac{\sqrt{d}}{N}\right)\leq\mu d^{\alpha/2}\cdot 2^{-M}+\mu\left(\frac{\sqrt{d}}{N}\right)^{\alpha}.
	\end{equation*}
	Set $\mu d^{\alpha/2}\cdot 2^{-M}=\mu\left(\frac{\sqrt{d}}{N}\right)^{\alpha}=\frac{\epsilon}{2}$. Then $M=\left\lceil\log_2\frac{\mu d^{\alpha/2}}{\epsilon}\right\rceil+1$, $N=\left\lceil\sqrt{d}\left(\frac{2\mu}{\epsilon}\right)^{1/{\alpha}}\right\rceil$.
\end{proof}

\begin{Corollary} $\label{Lp Lip result}$
\sloppy Let $p\in[1,+\infty)$. For any $f\in\mathcal{H}_{\mu}^{\alpha}([0,1]^d)$ and $\epsilon>0$, there exists a $\relu$-sine-$2^x$ network $\Phi$ with $\mathcal{L}(\Phi)=6, \mathcal{W}(\Phi)=\max\left\{2d\left\lceil\log_2(\sqrt{d}\left(\frac{3\mu}{\epsilon}\right)^{1/{\alpha}})\right\rceil ,2\left\lceil\log_2\frac{3\mu d^{\alpha/2}}{2\epsilon}\right\rceil+2\right\}$ such that
\begin{equation*}
\|f-\Phi\|_{L^p}\leq\epsilon.
\end{equation*}
\end{Corollary}
\begin{proof}
Let the parameters
\begin{equation*}
\mu d^{\alpha/2}\cdot 2^{-M}=\mu\left(\frac{\sqrt{d}}{N}\right)^{\alpha}=
\mu d^{\alpha/2}\left[1-(1-N\delta)^d\right]^{1/p}=\frac{\epsilon}{3}
\end{equation*}
in Theorem $\ref{Lp general result}$, we obtain the desired result.
\end{proof}

\begin{Corollary} \label{pointwise Lip result}
	\sloppy For any $f\in\mathcal{H}_{\mu}^{\alpha}([0,1]^d)$ and $\epsilon>0$, there exists a $\relu$-sine-$2^x$ network $\Phi$ with $\mathcal{L}(\Phi)=2d+6, \mathcal{W}(\Phi)=3^d\left[\max\left\{2d\left\lceil\log_2(\sqrt{d}\left(\frac{3\mu}{\epsilon}\right)^{1/{\alpha}})\right\rceil,2\left\lceil\log_2\frac{3\mu d^{\alpha/2}}{2\epsilon}\right\rceil+2\right\}+4\right]$ such that
\begin{equation*}
|f(\mathbf{x})-\Phi(\mathbf{x})|\leq \epsilon,\quad\mathbf{x}\in[0,1]^d.
\end{equation*}
\end{Corollary}
\begin{proof}
Applying Theorem $\ref{pointwise general result}$ and setting
\begin{equation*}
\mu d^{\alpha/2}\cdot 2^{-M}=\mu\left(\frac{\sqrt{d}}{N}\right)^{\alpha}=
\mu\delta^{\alpha} d=\frac{\epsilon}{3}
\end{equation*}
yields the result.
\end{proof}

The results in Corollary \ref{Lp Lip result} and
\ref{pointwise Lip result} show that  our proposed  $\relu$-sine-$2^x$ networks overcome the  curse of dimensionality {\color{blue} in approximation} on  H\"{o}lder  Class.
We should mention some related works on constructing  networks that break curse of dimensionality {\color{blue} in approximation}. In \cite{yarotsky2019phase}, to achieve accuracy $\epsilon$, a network with $\relu$ and any Lipschitz periodic activations with the total number of weights $\mathcal{O}\left(\log^2\frac{1}{\epsilon}\right)$ is built.
%In \cite{shen2020neural}, a nework with $2^x$, $\lfloor\cdot\rfloor$ and the indication activations is constructed to avoid curse of dimensionality. In the newest version of this preprint, continuous versions of these activations are explored and a network equipped with those continuous activations are built to break the curse of dimensionality.
In \cite{shen2020neural}, the authors  constructed a three hidden layer network that achieves the same approximation power as the $\relu$-sine-$2^x$  network constructed here. They  use    floor, $2^x$ and step functions as activation functions.  In the consideration of applying SGD for training, they propose using "continuous version" activation functions, i.e.,  utilizing   piecewise linear  functions to  approximate   the floor and step activation   functions. The resulting  "continuous version" network still enjoy  the super expressive power. However, the directional derivative  of  the  piecewise linear  functions may blow up since it  depends   on $1/\epsilon$, see Table $\ref{compare}$.
In recent work of Yarotsky \cite{yarotsky2021elementary}, network with  $\{\sin,\arcsin\}$ activation is constructed to approximate  continuous functions $f\in C\left([0,1]^d\right)$ with precision $\epsilon$. The main feature of the $\sin$-$\arcsin$ network is that   the size is $\mathcal{O}(d^2)$ and independent on $\epsilon$.  Hence such a network overcomes curse of dimensionality {\color{blue} in approximation}. We summarize  the  related works in Table $\ref{compare}$.
 %{\color{blue}Note that the constant in the width  may  depend on $d$ exponentially [40-42] or polynomially [34-35]}.
$\ref{compare}$.
\begin{table}[ht!]	
	\centering
	\begin{threeparttable}
		\caption{Previous works and our result ($\epsilon$ denotes the approximation accuracy)}
		\label{compare}
		\setlength{\tabcolsep}{0.5mm}{
%			change the table size
		\begin{tabular}{ccccc}
			\toprule
			Paper&Function class&Activation(s)&Depth&Width\cr
			\midrule
			\cite{yarotsky2017error}&$C^s\left([0,1]^d\right)$&$\relu$&$\mathcal{O}\left({\color{blue}\log d}\log\frac{1}{\epsilon}\right)$&\tabincell{c}{$\mathcal{O}\left({\color{blue}d^{s+1}\log d}\left(\frac{1}{\epsilon}\right)^{d/s}\log\frac{1}{\epsilon}\right)$ \\ (weights)}\cr
			\cite{yarotsky2019phase}&$C^s\left([0,1]^d\right)$($s\geq1$)&$\relu$, sine&$\mathcal{O}\left({\color{blue}d}\log\frac{1}{\epsilon}\right)$&$\mathcal{O}\left({\color{blue}d2^d}\log\frac{1}{\epsilon}\right)$\cr
			\cite{shen2020deep}&$\mathcal{H}_{\mu}^{\alpha}\left([0,1]^d\right)$&$\relu$, floor&$256d+3$&$\mathcal{O}\left({\color{blue}d}\epsilon^{\frac{-1}{2\alpha}}\right)$\cr
			\cite{shen2020neural}&$\mathcal{H}_{\mu}^{\alpha}\left([0,1]^d\right)$&$\rho_1,\rho_2,\rho_3$\tnote{1}&4&$\mathcal{O}\left({\color{blue}d\log d}\log\frac{1}{\epsilon}\right)$\cr
\cite{yarotsky2021elementary} & $C\left([0,1]^d\right)$	&	$\{\sin,\arcsin\}$&  \quad &  $\mathcal{O}(d^2)$ ($\mathrm{size}$)\cr	
this paper&$\mathcal{H}_{\mu}^{\alpha}\left([0,1]^d\right)$&$\relu$, sine and $2^x$&6&$\mathcal{O}\left({\color{blue}d\log d}\log\frac{1}{\epsilon}\right)$\cr
			\bottomrule
		\end{tabular}
	}
	\begin{tablenotes}
		\footnotesize
		\item[1] $ \varrho_{1, \delta}(x)=\left\{\begin{array}{ll}
n-1, & x \in[n-1, n-\delta], \\
(x-n+\delta) / \delta, & x \in(n-\delta, n],
\end{array} \quad \text { for any } n \in \mathbb{Z}\right.$, $\rho_2=3^x$,\textcolor{white}{a}
$ \varrho_{3}(x):=\tilde{\mathcal{T}}(\cos (2 \pi x)),$
$\widetilde{T}(x):=\left\{\begin{array}{ll}
0, & x \in\left(\cos \left(\frac{4 \pi}{9}\right), \infty\right), \\
1-x / \cos \left(\frac{4 \pi}{9}\right), & x \in\left[0, \cos \left(\frac{4 \pi}{9}\right)\right] \\
1, & x \in(-\infty, 0)
\end{array}\right..$
	\end{tablenotes}
	\end{threeparttable}
\end{table}

\begin{Remark}
Comparing with Corollary $\ref{Lp Lip result}$, there is an additional constant factor which is exponentially depending on dimension $d$ in the width of network in Corollary $\ref{pointwise Lip result}$. The factor $3^d$ is introduced by Proposition $\ref{expand to the whole region}$ since we want to expand Corollary $\ref{almost pointwise Lip result}$ to the whole region. Even so, a factor such as $\epsilon^{-d}$ does not appear in the depth and width of our network, which appears and leads to curse of dimensionality {\color{blue} in approximation} in many previous results of $\relu$ networks \cite{yarotsky2017error, suzuki2018adaptivity,shen2019deep}.

Furthermore, if we don't pursue pointwise accuracy and only interest in approximation in $L^p$ norm with $p\in [1,\infty)$,   Corollary $\ref{almost pointwise Lip result}$ and $\ref{Lp Lip result}$ provide   powerful and practical results, where  the depth are $6$ and the constant factors in width  only  depend on dimension $d$ at most in terms of $d\log d$.

 An important issue in practical learning tasks such as classification and regressions is to determine  the parameters in network $\Phi$ with  data.
 Since  $\Phi$ are (generalized) differentiable
   \cite{clarke1990optimization,berner2019towards},  we can utilize the  workhorse  SGD for training.
\end{Remark}

\begin{Remark}
We now compare our $\relu$-sine-$2^x$ network with the $\relu$-sine network appearing in \cite{yarotsky2019phase}. The depth and width of the former are $6$ and $\mathcal{O}\left(\log_2\frac{1}{\epsilon}\right)$, respectively (Corollary $\ref{Lp Lip result}$) while the depth and width of the latter are both $\mathcal{O}\left(\log_2\frac{1}{\epsilon}\right)$. Despite the width of two networks are of the same order, the depth of our network, a constant being independent of approximation error and dimension, is much less than the $\relu$-sine network in \cite{yarotsky2019phase}.
\end{Remark}

%\begin{Remark}
%We can see that the periodicity of the sine plays an significant role in our network construction. In fact the sine function can be replaced by other periodical functions.
%\end{Remark}
{\color{blue}
\section{Numerical Experiment}
In this section we will give two simple examples to show the approximation ability of the proposed deep neural network. Let
\begin{equation}
\begin{array}{c}
f_{1}=\prod_{1}^{d} \sin(\pi x_{i}), \;\;
f_{2}=\prod_{1}^{d} x_{i}^{2},
\end{array}
\end{equation}
with space dimension $d = 3$. The loss function is chosen as the least square:
\begin{equation}
\mathcal{L}_{i}(\Phi)=\mathop{\mathbb{E}}_{X\sim U([0,1]^{d})}\left[\Phi(X)^{2}-2\Phi(X)f_{i}(X)\right]\quad i=1,2,
\end{equation}
where $U([0,1]^{d})$ stands for the uniform distribution on $[0,1]^{d}$. Then we use the stochastic gradient decent (SGD) type algorithm to minimize the loss by taking samples $X_{i}\sim U([0,1]^{d})$. In our experiments, we use the Adam \cite{2014adam} optimizer with $1e5$ epochs and $1e5$ batch size. The learning rate is initially set to be $3e-3$ and reduced by $0.99$ in every $5000$ epochs.

The construction of the network is as Figure \ref{figure}. The width of the first layer is $4d$. The second layer and the first layer are fully connected in each dimension. Before the activation of the third layer, a truncation is applied to avoid the exponential blow-up. The width of the fourth layer is $8$. There are $12d+25$ neurons in total.

The result is shown in Figure \ref{figureresult}. The first row plots the landscape of $\Phi$ on the diagonal line of $[0,1]^{d}$ and its reference. The second row is the $L^{2}$ error of the approximation: $\|\Phi-f_{i}\|_{L^{2}}$. One may find that the SGD algorithm successfully minimizes the loss in the proposed neural network architecture.
\begin{figure}[H]
	\centering
	\includegraphics[width=0.4\textwidth]{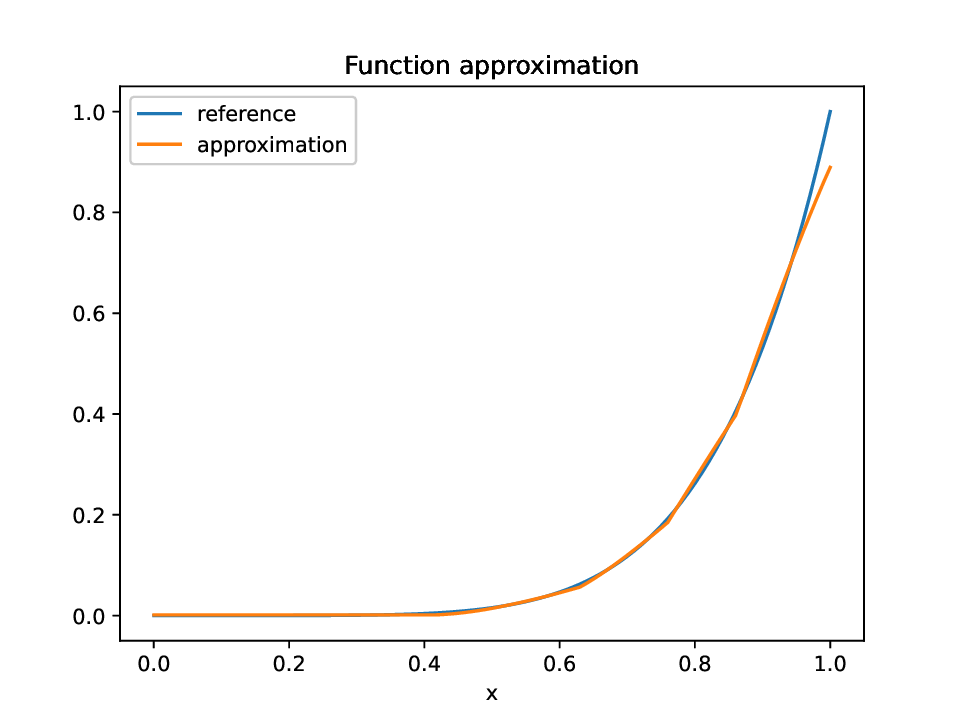}\quad
	\includegraphics[width=0.4\textwidth]{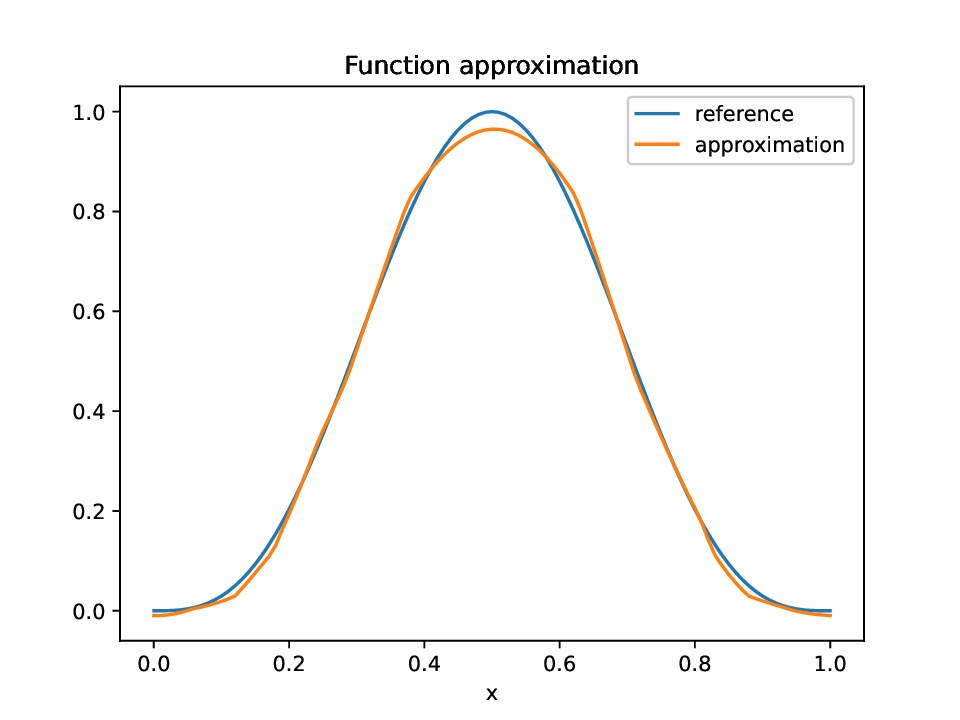}\quad
	\includegraphics[width=0.4\textwidth]{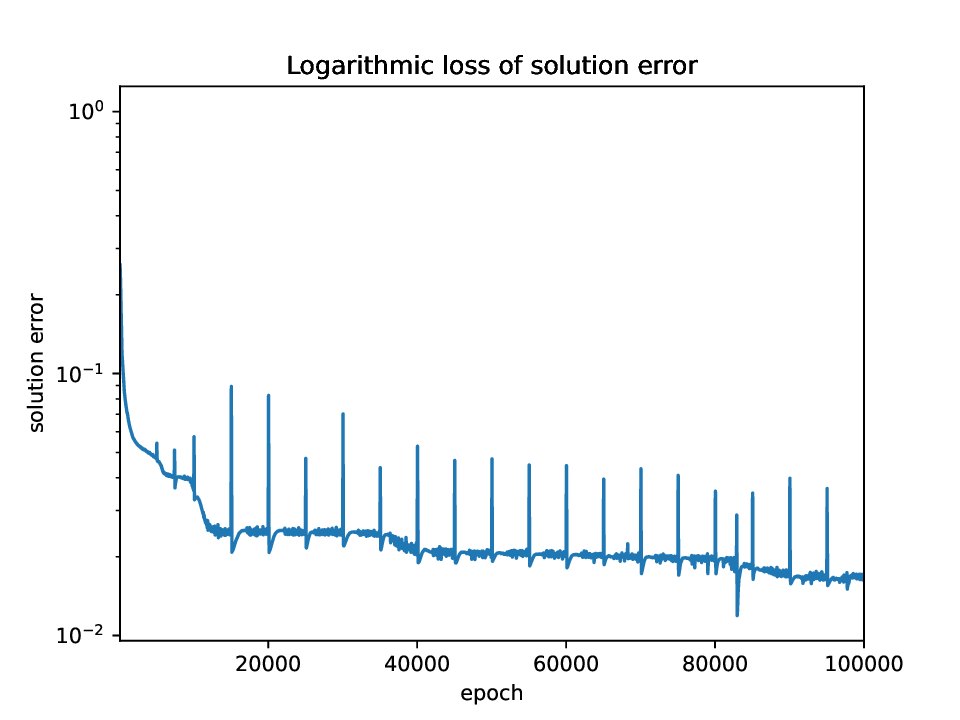}\quad
	\includegraphics[width=0.4\textwidth]{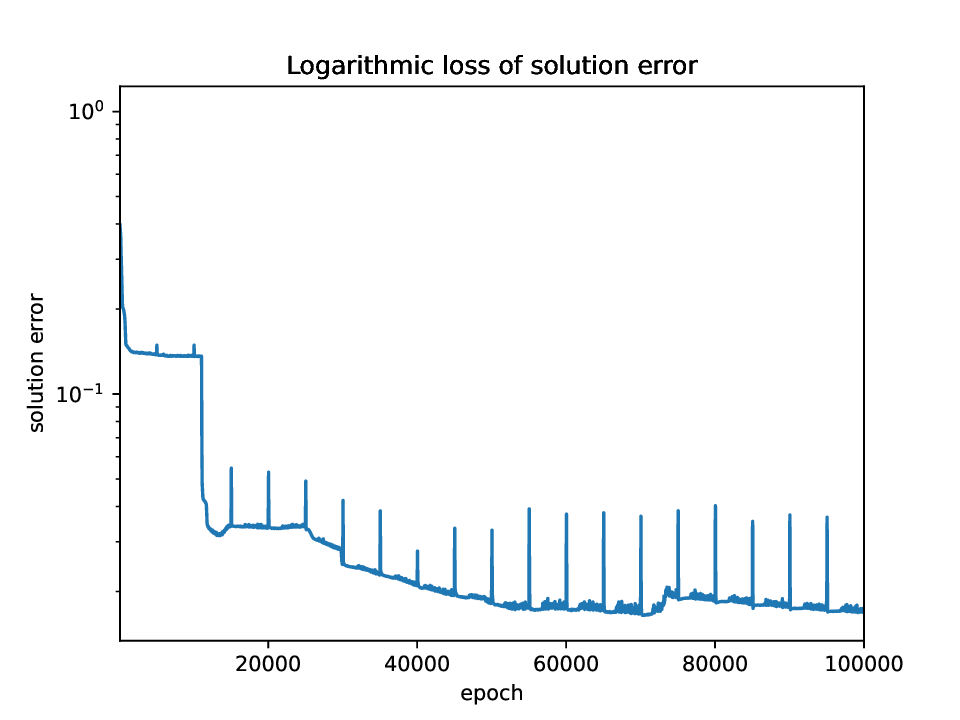}
	\caption{The numerical result of SGD optimization. The first row demonstrates the landscape of $\Phi$ on the diagonal line of $[0,1]^{d}$. The second row is the approximation error in $L^{2}$ norm.}
	\label{figureresult}
\end{figure}
}

\section{Conclusion}
In this paper, we construct neural networks with ReLU, sine and $2^x$ as activation functions that overcome the curse of dimensionality {\color{blue} in approximation} on the H\"{o}lder  continuous function class  defined on $[0,1]^d$.
The proposed $\relu$-sine-$2^x$  network functions   are (generalized) differentiable, enabling us to apply SGD to train in practical learning tasks.

There are several avenues for further study. First, due to the theoretical advantages established here, the practical performances of the $\relu$-sine-$2^x$ networks in real world applications  deserves careful evaluations.
Second, whether or not the generalization errors of $\relu$-sine-$2^x$ networks in supervised learning can break  the curse of dimensionality {\color{blue} in approximation} on number of samples   is also  of immense current interest.

\section*{Acknowledgement}
The authors would like to thank the anonymous referees for several
constructive comments, which have led to an improved presentation.
This work is supported by the National Key Research and Development Program
of China (No. 2020YFA0714200), by the National Science Foundation of China (No.
12125103, No. 12071362, No. 11971468, No. 11871474, No.11871385). The numerical calculations
have been done at the Supercomputing Center of Wuhan University.
\bibliographystyle{siam}
\bibliography{mybibfile}

\end{document}